\documentclass[utf8]{frontiersSCNS} 

\usepackage{url,hyperref,lineno,microtype,subcaption}
\usepackage{amsmath,amssymb,amsthm}
\usepackage{pdflscape}

\newtheorem{theorem}{Theorem}

\usepackage{bm} 
\usepackage{xspace}

\newcommand{\xmath}[1] {\ensuremath{#1}\xspace}
\newcommand{\blmath}[1] {\xmath{\bm{#1}}}

\newcommand{\Ab}{{\blmath A}}

\newcommand{\Db}{{\blmath D}}

\newcommand{\Ib}{{\blmath I}}

\newcommand{\Lb}{{\blmath L}}
\newcommand{\Mb}{{\blmath M}}

\newcommand{\Pb}{{\blmath P}}

\newcommand{\Rb}{{\blmath R}}
\newcommand{\Sb}{{\blmath S}}

\newcommand{\Ub}{{\blmath U}}

\newcommand{\Wb}{{\blmath W}}
\newcommand{\Xb}{{\blmath X}}
\newcommand{\Yb}{{\blmath Y}}

\newcommand{\ab}{{\blmath a}}

\newcommand{\eb}{{\blmath e}}

\newcommand{\hb}{{\blmath h}}

\newcommand{\pb}{{\blmath p}}

\newcommand{\rb}{{\blmath r}}

\newcommand{\xb}{{\blmath x}}
\newcommand{\yb}{{\blmath y}}

\newcommand{\Lc}{\mathcal{L}}
\newcommand{\Nc}{\mathcal{N}}

\newcommand{\Phib}{{\boldsymbol {\Phi}}}

\newcommand{\Rd}{{\mathbb R}}

\newcommand{\Dc}{{{\mathcal D}}}

\newcommand{\beq}{\begin{equation}}
\newcommand{\eeq}{\end{equation}}
\newcommand{\beqa}{\begin{eqnarray}}
\newcommand{\eeqa}{\end{eqnarray}}

\newcommand{\Lambdab}{\boldsymbol{\Lambda}}

\usepackage[onehalfspacing]{setspace}

\def\keyFont{\fontsize{8}{11}\helveticabold }
\def\firstAuthorLast{Byung-Hoon Kim and Jong Chul Ye} 
\def\Authors{Byung-Hoon Kim\,$^{1}$ and Jong Chul Ye\,$^{1,*}$}
\def\Address{$^{1}$Department of Bio and Brain Engineering, Korea Advanced Institute of Science and Technology (KAIST), Daejeon, Korea.}
\def\corrAuthor{Corresponding Author}

\def\corrEmail{jong.ye@kaist.ac.kr}

\begin{document}
\onecolumn
\firstpage{1}

\title[Graph Isomorphism Network rs-fMRI Analysis]{Understanding Graph Isomorphism Network for rs-fMRI Functional Connectivity Analysis}

\author[\firstAuthorLast ]{\Authors} 
\address{} 
\correspondance{} 

\extraAuth{}

\maketitle
\begin{abstract}

\section{}
Graph neural networks (GNN) rely on graph operations that include neural network training for various graph related tasks.
Recently, several attempts have been made to apply the GNNs to functional magnetic resonance image (fMRI) data.
Despite  recent progresses, a common limitation is its difficulty to explain the classification results in a neuroscientifically explainable way.
Here, we develop a framework for analyzing the fMRI data using the Graph Isomorphism Network (GIN), which was recently proposed as a powerful GNN for graph classification.
One of the important contributions of this paper is the observation that the GIN is a dual representation of convolutional neural network (CNN) in the graph space where the shift operation is defined using the adjacency matrix.
This understanding enables us to exploit CNN-based saliency map techniques for the GNN, which
we tailor to the proposed GIN with one-hot encoding, to visualize the important regions of the brain.
We validate our proposed framework using large-scale resting-state fMRI (rs-fMRI) data for classifying the sex of the subject based on the graph structure of the brain.
The experiment was consistent with our expectation such that the obtained saliency map show high correspondence with previous neuroimaging evidences related to sex differences.

\tiny
 \keyFont{ \section{Keywords:} graph neural networks, saliency mapping, functional neuroimaging, resting-state, explainable artificial intelligence} 
\end{abstract}

\section{Introduction}
Graphs provide an efficient way to mathematically model non-regular interactions between data in terms of nodes and edges \citep{bassett2009human, he2010graph,sporns2018graph}.
The network of the brain can be modeled as a graph consisting of ROIs as the nodes and their functional connectivity as the edges \citep{bassett2017network}.
In classical graph theoretic approaches, various graph metrics including local/global efficiency, average path length, and small-worldedness, are computed to analyze the brain networks \citep{wang2010graph}.
These metrics could be further used for group comparison to reveal the different network properties, providing insights to the physiological charactersitics and the disorders of the brain \citep{tian2011hemisphere,micheloyannis2006small}.
\par

Recently, there have been remarkable progresses and growing interests in Graph Neural Networks (GNNs), which comprise graph operations performed by deep neural networks (see the extensive survey in \citep{wu2019comprehensive}).
The GNNs are suitable for solving tasks such as node classification, edge prediction, graph classification, etc.
Usual GNNs typically integrate the features at each layer to embed each node features into a predefined next layer feature vector.
The integration process is implemented by choosing appropriate functions for aggregating features of the neighborhood nodes.
As one layer in the GNN aggregates its 1-hop neighbors, each node feature is embedded with features within its $k$-hop neighbors of the graph after $k$ aggregating layers.
The feature of the whole graph is then extracted by applying a readout function to the embedded node features.
\par

Considering the development of GNNs, it is not surprising that there are keen interests in applying GNNs to fMRI data analysis.
For example, some works have applied the GNN to classify one's phenotypic status based on the graph structure of the brain functional networks \citep{ktena2017distance, ktena2018metric, ma2018similarity, li2019graph, li2019graph2}.
Some other works employed the GNN to classify the subjects, not only based on the imaging data, but also including the non-image phenotypic data \citep{parisot2017spectral,parisot2018disease,he2019deep}.
Despite the early contribution of these works in applying the GNNs for fMRI analysis, there exists a common limitation in that they often fail to provide proper mapping of the ROIs for neuroscientific interpretation.
To overcome this limitation, there have been recent attempts to address the issue of neuroscientific interpretability by visualizing the important features of the brain \citep{arslan2018graph,duffy:MIDLAbstract2019a,li2019graph}.
These attempts involved saliency mapping methods of the GNNs, such as class activation mapping (CAM) \citep{zhou2016learning} to delineate the important features, as demonstrated in \citep{arslan2018graph}.
\par

Here we revisit the Graph Isomorphism Network (GIN) \citep{xu2018powerful}, which was recently proposed to implement  Weisfeiler-Lehman (WL) graph isomorphism test  \citep{shervashidze2011weisfeiler} in a neural network.
Our classification results on sex classification confirmed that GIN method can provide more powerful classification performance, but
the direct calculation of the graph saliency map was not clear.

Therefore, another important contribution of this work is to show that
while GIN is similar to spectral-domain approaches such as the graph convolutional network (GCN) in learning
the spectral filters from graphs, GIN can be considered as a dual representation
of the convolutional neural network (CNN) with two-tab convolution filter in the graph space where the adjacency matrix is defined as a generalized shift operation.
With this generalization, we can employ one of the most widely used saliency map visualization technique in CNN, called the gradient-weighted class activation mapping (Grad-CAM) \citep{selvaraju2017grad} that can be applied to any CNN architecture at any layer.
We further found that to visualize the important brain regions that are related to a certain phenotypic difference,
Grad-CAM should be calculated at the input layer and the one-hot encoding of the graph node is ideally suitable
for such saliency map visualization.

Experimental results on sex classification confirm that our method can provide more accurate classification performance and
better interpretability of the classification results in terms of saliency maps, which provide some new insights to the topic of sex differences on the resting-state fMRI (rs-fMRI).
\par

\subsection{Mathematical Preliminaries}
We denote a graph $G=(V,E)$ with a set of vertices $V(G)=\{1,\cdots, N\}$ with $N:=|V|$  and edges $E(G)=\{e_{ij}\}$, where an edge $e_{ij}$ connects vertices $i$ and $j$ if they are adjacent or neighbors.
The set of neighborhoods of a vertex $v$ is denoted by $\Nc(v)$.
For weighted graphs, the edge $e_{ij}$ has a real value.
If $G$ is an unweighted graph, then $E$ is a sparse matrix with elements of either 0 or 1.
\par

When analyzing the fMRI data, the functional connectivity between two regions of the brain is often computed from the Pearson correlation coefficient between the fMRI time series.
Specifically, the Pearson correlation coefficient between the fMRI time series $\yb_i$ at
the vertex $i$ and the fMRI time series $\yb_j$ at the vertex $j$ is given by $$R_{ij} = \frac{\mathrm{Cov}(\yb_i,\yb_j)}{\sigma_{\yb_i}\sigma_{\yb_j}} \in \Rd^{N \times N}$$ where $\mathrm{Cov}(\yb_i,\yb_j)$ is the cross covariance between $\yb_i$ and $\yb_j$, and $\sigma_{\yb_i}$ denotes the standard deviation of $\yb_i$.
Unweighted graph edge can be derived from the functional connectivity by thresholding
the correlation coefficients by  a certain threshold.

For a simple unweighted graph with vertex set $V$, the adjacency matrix is a square $|V| \times |V|$ matrix $\Ab$ such that its element $A_{uv}$ is one when there is an edge from vertex $u$ to vertex $v$, and zero when there is no edge.
For the given adjacency matrix $\Ab$, the graph Laplacian $\Lb$  and its normalized version $\Lb_n$ are then defined by
\begin{align}\label{eq:Laplacian}
\Lb := \Db-\Ab, &\quad \Lb_n = \Ib -\Db^{-\frac{1}{2}}\Ab\Db^{-\frac{1}{2}}
\end{align}
where $\Db$ is the degree matrix with  the diagonal element
\begin{align}\label{eq:degree}
D_{uu}=d(u) = \sum_{v} A_{uv} \ ,
\end{align}
and zeros elsewhere.

Graph Laplacian is useful for signal processing on a graph \citep{shuman2013emerging,ortega2018graph,huang2018graph}.
More specifically, the graph convolution for
real-valued functions on the set of the graph's
vertices, $\xb,\yb : V \mapsto \Rd^{|V|}$ is often defined by
\begin{align}\label{eq:conv}
\xb \ast_G \yb = \Ub \left( \Ub^\top \xb \odot \Ub^\top \yb \right)
\end{align}
where the superscript $^\top$ denotes the adjoint operation,
$\Ub$ is the matrix composed of singular vectors of the normalized graph Laplacian, i.e.
\begin{align}
\Lb_n = \Ub\Lambdab\Ub^\top
\end{align}
where $\Lambdab$ denotes the diagonal matrices with the singular values, which is often
referred to as the graph spectrum.

\subsection{Graph Neural Networks}

The goal of GNNs for the graph classification task is to learn a nonlinear mapping $g$ from a graph to a feature vector:
\begin{align}\label{eq:g}
  g: G \mapsto \pb_{G},
  \end{align}
where $\pb_{G}$ is a feature vector of the whole graph $G$ that helps predicting the labels of the graph.
Recent perspective distinguishes the GNNs into two groups based on the neighborhood aggregating schemes \citep{wu2019comprehensive}.
First group is the spectral-based convolutional GNNs (spectral GNN).
This group of GNNs are inspired by the spectral decomposition of the graphs, and aim to approximate the spectral filters in each aggregating layers \citep{bruna2013spectral, kipf2016semi}.
The other group of GNNs are the spatial-based convolutional GNNs (spatial GNN).
They do not explicitly aim to learn spectral features of graph, but rather implement the neighborhood aggregation based on the nodes' spatial relations.
Some well-known examples of the spatial GNNs are the Message Passing Neural Network (MPNN) \citep{gilmer2017neural} and the GIN \citep{xu2018powerful}.
In this section, we provide a brief review of the these approaches to  understand their relationships.

\par

Spectral GNNs are based on the graph convolution relationship \eqref{eq:conv}, in which
$ \Ub^\top \yb$ is replaced by the parameterized graph spectrum  $\hat\yb := \Ub^\top \yb$:
\begin{align*}
\xb \ast_G \yb = \Ub\left(\hat\yb\odot\Ub^\top\xb\right)
\end{align*}
More specifically, the graph convolutional layer of the spectral GNN is then implemented as follows:
\begin{align}\label{eq:spectralconv}
\xb_i^{(k)}=\sigma\left(\sum\nolimits_{j}\Ub\Yb^{(k)}_{i,j}\Ub^\top\xb_{j}^{(k-1)}\right)
\end{align}
where $\sigma(\cdot)$ is an element-by-element nonlinearity,
  $\xb_i^{(k)}$ is the graph signal at the channel $i$ of $k$-th layer and
$\Yb^{(k)}_{i,j}$ is a diagonal matrix that parameterized the graph spectrum $\hat\yb$ with learnable parameters.

To realize these ideas,
GCN was proposed as the first-order approximation of the spectral GNN \citep{hammond2011wavelets,kipf2016semi}.
Specifically, the authors of \citep{kipf2016semi} showed that the first order-approximation
of the Chebyshev expansion of the spectral convolution operation
can be implemented as the spatial domain convolution:
 \begin{equation}\label{eq:gcn_matrix}
  \Xb^{(k)}=\sigma \Bigl( \tilde{\Db}^{-\frac{1}{2}} \tilde{\Ab} \tilde{\Db}^{-\frac{1}{2}}  \Xb^{(k-1)} \Wb^{(k)} \Bigr) \in \Rd^{N \times C^{(k)}}.
\end{equation}
where $\tilde{\Ab} =  \Ab + \Ib $ is the adjacency matrix assuming the recurring loop,
$\tilde\Db$ is a degree matrix of $\tilde\Ab$,
and
\begin{align}\label{eq:X}
  \Xb^{(k)} &=  \begin{bmatrix} \xb_{1}^{(k)} &\cdots & \xb_{C^{(k)}}^{(k)}  \end{bmatrix}  \in \Rd^{N\times C^{(k)}}
\end{align}
denotes the $C^{(k)}$ channel signals at the
$k$-th layer.
This implies that GCN implements  the node feature with its neighborhoods  by mapping through a layer-specific learnable weight matrix $\Wb^{(k)}$ and nonlinearity $\sigma$.

Unlike the spectral GNN,
spatial-based methods define graph convolutions
based on a node’s spatial relations.
More specifically,  this opertion is
generally composed of the $\mathtt{AGGREGATE}$, and $\mathtt{COMBINE}$ functions:
\begin{equation*}\label{eq:aggregate}
  \ab_{v}^{(k)}=\mathtt{AGGREGATE}^{(k)}\Bigl( \Bigl\{ \pb_{v}^{(k-1)}:u \in \Nc (v) \Bigr\} \Bigr) ,
\end{equation*}
\begin{equation*}\label{eq:combine}
  \pb_{v}^{(k)}=\mathtt{COMBINE}^{(k)} \Bigl( \pb_{v}^{(k-1)}, \ab_{v}^{(k)} \Bigr) ,
\end{equation*}
where $\pb_v^{(k)} \in \Rd^{C^{(k)}}$ denotes the $k$-th layer feature vector at the $v$-th node.
In other words, the $\mathtt{AGGREGATE}$ function collects features of the neighborhood nodes to extract aggregated feature vector $\ab_{v}^{(k)}$ for the layer $k$, and
$\mathtt{COMBINE}$ function then combines the previous node feature $\pb_{v}^{(k-1)}$ with aggregated node features $\ab_{v}^{(k)}$ to output the node feature of the current $k$-th layer $\pb_{v}^{(k)}$.
After this spatial operation, the mapping \eqref{eq:g}
is defined by
\begin{equation*}\label{eq:readout}
  \pb_{G}=\mathtt{READOUT} \Bigl( \Bigl\{ \pb_{v}^{(k)} | v \in G \Bigr\} \Bigr) .
\end{equation*}
Moreover, the $\mathtt{AGGREGATE}$  and $\mathtt{COMBINE}$ share the similar idea
of information propagation/message passing on graphs \citep{wu2019comprehensive}.

In particular, GIN was proposed by \citep{xu2018powerful} as a special case of spatial GNN suitable for graph classification tasks.
The network implements the aggregate and combine functions as the sum of the node features:
\begin{equation}\label{eq:gin}
  \pb_{v}^{(k)}=\mathtt{MLP}^{(k)}\Bigl((1+\epsilon^{(k)}) \cdot \pb_{v}^{(k-1)}+\sum\nolimits_{u \in \Nc(v)}\pb_{u}^{(k-1)}\Bigr) \in \Rd^{C^{(k)}},
\end{equation}
where $\epsilon^{(k)}$ is a learnable parameter, and $\mathtt{MLP}$ is a multi-layer perceptron with nonlinearity.
For graph-level readout,
the embedded node features of every layers are summed up and then concatenated to obtain the final graph feature $\pb_{G}$ as in \citep{xu2018powerful,xu2018representation},
\begin{align}
  \pb_{G}^{(k)} &= \mathtt{sum}(\pb_{0}^{(k)}, \pb_{1}^{(k)}, ..., \pb_{N}^{(k)})  \label{eq:gin_sum}  \\
  \pb_{G} &= \mathtt{concatenate}(\{\pb_{G}^{(k)}\}|k=0,1,...,K) \label{eq:gin_readout}.
\end{align}
The authors of \citep{xu2018powerful} argue that the proposed network architecture can learn injective mapping of the function $g$, which makes the model to be possibly as powerful as the WL test for graph classification tasks \citep{weisfeiler1968reduction,shervashidze2011weisfeiler,xu2018powerful}.

\section{Theory}

In this section, we mathematically show that the GIN is a dual representnation of CNN on the graph space where
the adjacency matrix is defined as a generalized shift operation.
Along with this finding, we further propose a method for applying the GIN to the rs-fMRI data for graph classification and analysis.

\subsection{GIN as a generalized CNN on the graph space}
Note that the GIN processing \eqref{eq:gin} can be decomposed as
\begin{align}
  \pb_{v}^{(k)}&=\mathtt{MLP}^{(k)}(\rb_v^{(k)})~~ \in \Rd^{C^{(k)}},\quad v=1,\cdots, N,
 \end{align}
 where
 \begin{align}
 \rb_v^{(k)} &= c^{(k)} \pb_{v}^{(k-1)}+\sum\nolimits_{u \in \Nc(v)}\pb_{u}^{(k-1)} \label{eq:gin2}  \\
 &=   \underbrace{\begin{bmatrix} \pb_{1}^{(k-1)} &\cdots & \pb_{N}^{(k-1)} \end{bmatrix}}_{\Pb^{(k-1)}} \left(c^{(k)}\Ib+\Ab\right)_{:,v} \in \Rd^{C^{(k-1)}}
\end{align}
where $c^{(k)}:=1+\epsilon^{(k)}$ and
 $\Ab$ is the adjacency matrix
and $\Mb_{:,v}$ denotes the $v$-th column of a matrix $\Mb$.
This operation is performed for $k=1,\cdots, K$.

One of the most important observations is that the
feature matrix  $\Pb^{(k-1)}$ is closely
related to the signal matrix $\Xb^{(k-1)}$ in \eqref{eq:X}.
More specifically, we have the following dual relationship:
\begin{align}\label{eq:equiv}
\Xb^{(k-1)}= \Pb^{(k-1)\top}
\end{align}
 Then, using the observation that $c^{(k)}\Ib+\Ab$ is self-adjoint,
the matrix representation of \eqref{eq:gin2} can be converted to a dual representation:
\begin{align}
  \Xb^{(k)}
  &=\sigma\left( (c^{(k)}\Ib +\Ab)\Xb^{(k-1)}  \Wb^{(k)}\right) \in \Rd^{N \times C^{(k)}}
  \label{eq:gin_matrix}
\end{align}
where
$\Wb^{(k)}\in \Rd^{C^{(k-1)}\times C^{(k)}}$ denotes the fully connected network weight from the MLP.
Eq.~\eqref{eq:gin_matrix} shows that
aside from
the  iteration dependent $\epsilon^{(k)}$,
the main difference of GIN from GCN is the presence of the $(c^{(k)}\Ib +\Ab)$
instead of the normalized adjacency matrix $\tilde\Ab$.
This implies that GIN can be considered as an extension of the GCN as a first order approximation of the spectral GNN using the unnormalized graph Laplacian.

However, another important contribution of this paper is that the difference is not a minor variation,
but that it implies an imporant difference between the two approaches.
More specifically, by exploring the role of $c^{(k)}$ in \eqref{eq:gin_matrix},
Theorem~\ref{thm:gin} shows that \eqref{eq:gin_matrix} is a dual representation
 of the two tab convolutional neural network without pooling layer
on the graph spaces, where the adjacency matrix is defined as a shift operation.

\begin{theorem}\label{thm:gin}
 The GIN iteration in \eqref{eq:gin2} or \eqref{eq:gin_matrix} is a dual representation of a
 CNN without pooling  layers using two-tab filter
 on the  graph space, where the adjacency matrix $\Ab$ is defined as a shift operation.
 \end{theorem}

\begin{proof}
To understand this claim, we first revisit the classical CNN for the 1-D signal.
A building block for the CNN is the following multi-channel convolution \citep{ye2019understanding}:
 \begin{align}\label{eq:Conv}
\xb_i^{(k)}&= \sigma\left(\Phib^{\top} \sum_{j=1}^{C^{(k-1)}}\left(\xb_j^{(k-1)}\circledast \hb_{i,j}^{(k)}\right)\right)
\end{align}
where $C^{(k)}$ is the number of channels at the $k$-th layer,
 $\xb_i^{(k)}$ denotes the $i$-th channel signal at the $k$-th layer, and $\hb_{i,j}^{(k)}$ is the
convolution filter that convolves with $j$-th input channel signal to produce $i$-th channel output.
Finally,  $\Phib^\top$ denotes the matrix that represent the pooling operation.

Suppose that the convolution filter $\hb_{i,j}^{(k)}$ has two tabs. Without loss of generality,  the filter
can be represented by
$$\hb_{i,j}^{(k)} = \begin{bmatrix}   c^{(k)} w_{i,j}^{(k)} & w_{i,j}^{(k)} \end{bmatrix}^\top \in \Rd^{2} $$
for some constant $c^{(k)}, w_{i,j}^{(k)}$.
Then,
the convolution operation can be simplified as
\begin{align*}
\xb_j^{(k-1)}\circledast \hb_{i,j}^{(k)} =  c^{(k)} w_{i,j}^{(k)}  \xb_j^{(k-1)} + w_{i,j}^{(k)}\Sb\xb_j^{(k-1)}
\end{align*}
where $\Sb$ is the shift matrix defined by
\begin{align}\label{eq:shift}
\Sb = \begin{bmatrix} 0 & 0 & \cdots & 0 & 1 \\ 1 & 0 & \cdots &  0 & 0 \\ \vdots & \vdots & \ddots & \vdots & \vdots \\
0 & \cdots & \cdots & 0 & 0 \\
0 & \cdots & \cdots & 1 & 0 \end{bmatrix}
\end{align}
if we assume the periodic boundary condition.
Accordingly,  for the cases of a CNN with no pooling layers,  i.e. $\Phib^\top=\Ib$,
\eqref{eq:Conv} with the two-tab filter can be represented in the following matrix form:
\begin{align}\label{eq:cnn_matrix}
\Xb^{(k)}&= \sigma\left(\left( c^{(k)} \Xb^{(k-1)}
+\Sb \Xb^{(k-1)}\right)\Wb^{(k)} \right)
\end{align}
where
\begin{align*}
\Xb^{(k)} &= \begin{bmatrix} \xb_1^{(k)} & \cdots & \xb_{C^{(k)}}^{(k)} \end{bmatrix} \in \Rd^{N\times C^{(k)}}\\
\Wb^{(k)} &=\begin{bmatrix} w_{1,1}^{(k)} & \cdots & w_{C^{(k)},1}^{(k)} \\ \vdots & \ddots & \vdots \\ w_{1,C^{(k-1)}}^{(k)} & \cdots & w_{C^{(k)},C^{(k-1)}}^{(k)} \end{bmatrix} \in \Rd^{C^{(k-1)}\times C^{(k)}}
\end{align*}
By inspection of the dual representation of GIN in \eqref{eq:gin_matrix} and the CNN operation \eqref{eq:cnn_matrix}, we can
see that the only difference of  \eqref{eq:gin_matrix} is the adjacency matrix $\Ab$  instead of the shift matrix $\Sb$ in \eqref{eq:cnn_matrix}.
Therefore, we can conclude that the GIN is a dual representation of CNN with  two tab filter in the graph space where adjacency matrix is defined as a shift operation.
\end{proof}

Note that the identification of the adjacency matrix as a generalized shift operation is not our own invention, but rather it is a classical observation in graph signal processing literature \citep{shuman2013emerging,ortega2018graph,huang2018graph}.
Accordingly, Theorem~\ref{thm:gin} confirms that the insight from the classical signal processing plays an important role in understanding the GNN.
Based on this understanding, we can now provide a dual space insight of the GIN operations in \eqref{eq:gin_sum} and \eqref{eq:gin_readout}.
More specifically, \eqref{eq:gin_sum} can be understand as sum-pooling operation, since we have
\begin{align}
\left(\pb_G^{(k)}\right)^\top = \Phib_{\text{sum}}^\top \Xb^{(k)},&
\end{align}
where the pooling matrix $\Phib_{\text{sum}}^\top$ is given by
\begin{align}\label{eq:sumpool}
\Phib_{\text{sum}}^\top=\begin{bmatrix} 1 & \cdots & 1 \end{bmatrix}.
\end{align}
Then, \eqref{eq:gin_readout} is indeed the multichannel concatenation layer from the
pooled feature at each layer as shown in Figure~\ref{fig:scheme}.
Therefore, the GIN operations can be understood as a  dual representation of  CNN
classifier on the graph signal space where the shift operation is defined by the adjacency matrix.
In fact, CNN and GIN  differs in their definition of the shift operation as shown in
Figure~\ref{fig:scheme} and Figure~\ref{fig:shift}.
We provide an exemplar GIN operation for a more expressive explanation in the Figure~\ref{fig:smallgraph_example} and the Appendix.

\subsection{Saliency Map of GIN}
\label{sec:theory_gin_saliency}

Thanks to the  mathematical understanding of the similarity between the GIN and the CNN, we can now readily use the
saliency map  techniques for the CNNs
to  visualize important brain regions.
For example, \citep{arslan2018graph} used the CAM to visualize the graph saliency map.
Instead, we propose to visualize the salient regions based on the Grad-CAM, which is a generalized version of the CAM without the restriction of the need of the global average pooling layer \citep{selvaraju2017grad}.
Specifically, the Grad-CAM saliency map at the $k$-th layer GIN  can be calculated by
\begin{align}
\mathcal{S}(k)=\sum_{j=1}^N \alpha_j^{(k)} \xb_j^{(k)}
\end{align}
where
\begin{align}
\alpha_j^{(k)} =\sum_{i=1}^N \frac{\partial y}{\partial X_{ij}^{(k)}}
\end{align}
where $X_{ij}^{(k)}$ is the $(i,j)$-th element of $\Xb^{(k)}$ or $i$-th element of $\xb_j^{(k)}$.
Since we are interested in the input node contribution for the classification,
we found that the meaningful Grad-CAM saliency map should be calculated at the input layer, i.e. $k=0$, in which case the final representation becomes
much simpler:
\begin{align}\label{eq:saliency}
\mathcal{S}(0) &=\sum_{j=1}^N \alpha_j^{(0)} \xb_j^{(0)} =  \sum_{j=1}^N \alpha_j^{(0)}\eb_j  \notag\\
&= \begin{bmatrix} \sum\limits_{i=1}^{N} \frac{\partial{y}}{\partial X^{(0)}_{i1}}  & \cdots &
\sum\limits_{i=1}^{N} \frac{\partial{y}}{\partial X^{(0)}_{iN}} \end{bmatrix}^\top \in \Rd^N
\end{align}
where the second equality comes from that  $\xb_j^{(0)}$ is one-hot vector, i.e. $\xb_j^{(0)}=\eb_j$, where $\eb_j$ has one at the $j$-the elements whereas all the other elements are zero,
and the last equality comes from
\begin{align}
 \alpha_j^{(0)}=\sum_{i=1}^N \frac{\partial y}{\partial X_{ij}^{(0)}}
\end{align}
Note that in contrast to
 CAM \citep{zhou2016learning}  as in \citep{arslan2018graph} where sensitivity should be calculated with respect to the last layer,
our approach using Grad-CAM provides a direct link
from the input nodes to the final classification.
Using experimental data,  we will show that the resulting saliency map can quantify the sensitivity with respect to the node geometry, which
provide a  neuroscientific information about the relative importance of the each ROIs related to the class features.

\section{Materials and Methods}

Based on the aforementioned understanding of the GIN, we proceed to apply the GIN to the rs-fMRI data for classification of the subjects' sex
and provide neuroscientific interpretation.
The Figure~\ref{fig:scheme} provides schematic illustration of the proposed analysis pipeline.

\subsection{Data Description and Preprocessing}
The rs-fMRI data was obtained from the Human Connectome Project (HCP) dataset S1200 release \citep{van2013wu}.
The data was acquired for two runs of two resting-state session each for 15 minutes, with eyes open fixating on a cross-hair (TR=720ms, TE=33.1ms, flip angle=$52^\circ$, FOV=$208 \times 180$mm, slice thickness=2.0mm).
Of the total 4 runs, we used the first run of the dataset.
Preprocessing of the fMRI volume time-series included gradient distortion correction, motion correction, and field map preprocessing, followed by registration to T1 weighted image.
The registered EPI image was then normalized to the standard MNI152 space.
Finally, FIX-ICA based denoising was applied to reduce non-neural source of noise in the data \citep{salimi2014automatic,griffanti2014ica}.
Details of the HCP preprocessing pipeline is referred to \citep{glasser2013minimal}. \par

From the preprocessed HCP dataset, rs-fMRI scans of 1094 subjects were obtained from the project.
To further minimize the unwanted effect of head motion on model training, we discarded the subject scans with framewise displacement (FD) over 0.3mm at any time of the scan.
The FD was computed with \texttt{fsl\_motion\_outliers} function of the FSL \citep{jenkinson2012fsl}.
There were 152 discarded scans from filtering out with the FD, and 942 scans were left.
The 942 scans consisted of data from 531 female subjects and 411 male subjects.
We paired each scan with the sex of the corresponding subject as an input-label for training the neural network. \par

\subsection{Graph Construction from Preprocessed Data}
The ROIs are defined from the cortical volume parcellation by \citep{schaefer2017local}.
We used the 400 parcellations as in \citep{kashyap2019individual,weis2019sex}.
Semantic region labels (e.g., Posterior cingulate cortex) and functional network labels (e.g., Default mode) corresponding to every parcels are provided with the dataset \citep{schaefer2017local}.
Vertices are defined as one-hot vectors encoding the semantic region labels of the whole 400 ROIs.
It can be said that no actual signal from the fMRI blood oxygen level dependency (BOLD) activity is represented in the vertex of the constructed graph.\par

To define the edges, functional connectivity matrix was constructed as follows.
First, mean time-series of cortical parcels were obtained by averaging the preprocessed fMRI data voxels within each ROIs.
Functional connectivity is defined as the correlation coefficient of the pearson's correlation between the time-series of the two voxels.
Thus, the connectivity matrix is constructed by computing the pearson's correlation cofficient between every other ROIs.
Derivation of the mean time-series and the connectivity matrix was performed with the MATLAB toolbox GRETNA \citep{wang2015gretna}.
To derive an undirected, unweighted graph from the connectivity matrix, we threshold the connectivity matrix with sparsity by selecting the top $M$-percentile elements of the connectivity matrix as connected, and others unconnected. \par

\subsection{Training details}
\label{sec:training}
All following experiments are conducted with PyTorch 1.4.0.
We used the GIN (Eq.~\eqref{eq:gin}) for our classification experiment.
The concatenated graph features from all $K$ layers $\pb_{G}$ in \eqref{eq:gin_readout} is mapped to the classifier output
$\yb=[y[1],\cdots,y[c]]^\top$
 for predicting the one-hot vector encoded ground-truth label of the graph $\yb_{\text{gt}} =[y_{\text{gt}}[1],\cdots, y_{\text{gt}}[c]]^\top$, where $y_{\text{gt}}[i]\in \{0,1\}$ and $c$ is a set of all possible class labels.
Note that we omit the graph feature from the 0-th layer when concatenating since it is the same one-hot embedding of each pre-defined ROIs which have no difference between the subjects.
One-dimensional batch normalization was applied after each layers of the network followed by the ReLU activation.
The GIN is then trained to minimize the cross-entropy loss $\Lc_\mathtt{xent}$:
\begin{equation}\label{eq:loss_xent}
  \Lc_\mathtt{xent} = -E\left[\sum_{i=1}^{c} y_{\text{gt}}[i] \cdot \log({y[i]})\right]
\end{equation}
where the expectation is taken over the
training data.
For the sex classification in this paper, the classifier is binary, so we use $c=2$.

Deep Graph Infomax (DGI) was introduced in \citep{velivckovic2018deep} as an unsupervised method for the representation learning of the graph.
The DGI learns the node representation by maximizing the mutual information between the node feature vectors $\pb_{v}$ and the corresponding graph feature $\pb_{G}$.
A discriminator $\Dc$ that takes a pair of a node feature vector and a graph feature as input is trained to discriminate whether the two embeddings are from the same graph:
\begin{equation}\label{eq:loss_infomax}
  \Lc_\mathtt{Infomax} = \sum \log\mathcal{D}(\pb_{v}, \pb_{G}) + \sum \log(1-\mathcal{D}(\tilde{\pb_{v}}, \pb_{G})).
\end{equation}
Here, $\tilde{\pb_{v}}$ is a corrupted node feature vector, which is usually obtained by randomly selecting a node feature vector from another sample in the minibatch \citep{velivckovic2018deep}.
The DGI was first proposed as an unsupervised representation learning method, but \citep{li2019graph2} has made use of the DGI as a regularizer for the graph classification task.

Following the work by \citep{li2019graph2}, we added the DGI loss as a regularizer with the expectation that maximizing the mutual information between the node features and the graph features can help extract better representation of the graph.
Thus, the final loss function is defined as:
\begin{equation}\label{eq:loss}
  \Lc = \Lc_\mathtt{xent} + \lambda \cdot \Lc_\mathtt{Infomax},
\end{equation}
where $\Lc_\mathtt{xent}$ is the cross entropy loss  in \eqref{eq:loss_xent} and $\Lc_\mathtt{Infomax}$ is defined in \eqref{eq:loss_infomax}, respectively.
In this paper, we coin the term {\em Infomax regularization} indicating the regularizer $\Lc_\mathtt{Infomax}$.
To train the network, the Adam optimizer was used for 150 epochs of training with the learning rate of 0.01.
Learning rate was decayed by the scale of 0.8 after every 5 epochs of training.
We performed 10-fold cross-validation of the 942 graphs following \citep{varoquaux2017assessing}.
The final model hyperparameters are reported in the section \ref{sec:results_classification} based on the hyperparameter tuning experiments.
\par

\subsection{Comparative Study}

To investigate the optimality of the proposed method, we performed comparative study with other methods.
The first comparative study was performed to ensure the classification capability of our proposed method over other recent ones.
Specifically, we re-implemented and evaluated the performance of the GCN-based method by \citep{arslan2018graph} on our HCP dataset to serve as the baseline.
Additionally, we compared the results of sex classification accuracy on the same HCP dataset reported by \citep{zhang2018functional, weis2019sex}.
Second compararative study was to find the optimal hyperparameter of our proposed method.
We performed several hyperparameter tuning experiments which includes varying the level of sparsity, regularization coefficient $\lambda$, number of layers, number of hidden units, learning rate, and the dropout rate with the same dataset and the same GIN model.
Lastly, we compared the classification performance when the input features were not encoded in one-hot vectors.
Instead of embedding the input feature as a one-hot vector of each parcellation ROIs, we embedded the input features as mean BOLD activation of the ROI or its centroid coordinates \citep{ktena2017distance,ktena2018metric,li2019graph,li2019graph2}, and trained the proposed model with same model hyperparameters.
The centroid coordinates are defined as a three-dimensional vector with each vector element representing the location of the axis R, A, and S.
To exclude the possibility that the difference in classification performance comes from the first layer width of the model, we performed an additional experiment that the embedded centroid coordinate node features are first linearly mapped into the same dimension as in the one-hot encoded case, which is 400.

\subsection{Saliency Mapping}

The proposed saliency mapping was applied for visualizing the brain regions that are related to each class of sexes.
We computed the saliency map using  \eqref{eq:saliency} for each test subject.
To obtain the group-level map, each subject-level saliency map was averaged across all subjects, and then was normalized to the range $[0.0, 1.0]$.
Here, we specifically focus on the regions within the top 5-percentile values, which correspond to top 20 regions of the 400.
To clarify the validity and advantages of our method, we compare the robustness and mapping results with the CAM-based saliency mapping method by \citep{arslan2018graph}.
We evaluate how many top 5-percentile salient regions from only a subset of the subject groups match those from the whole group to demonstrate the robustness of the methods.
Specifically, we compute the ratio of matching top salient regions between the maps derived by aggregating the full fold results and the maps derived from each fold of the cross-validation tests.
Each cross validation fold consisted of around one tenth (n=95 or n=94) of the whole subjects (n=942).
The final robustness is calculated as the average of the matching ratios of the each ten fold maps.
Comparison of the full fold aggregated result and the five fold aggregated result (n=470 or n=472) was additionally done.

\section{Results}\label{sec:results}

\subsection{Classification results}
\label{sec:results_classification}
The classification accuracy, precision, and recall are reported in Table \ref{table:comparison} along with other methods on the same first run of the HCP dataset.
Highest accuracy of 84.61\% was achieved by the proposed method, whereas the baseline GCN-based method achieved 83.98\% accuracy.
Other recent approaches with non GNN-based methods reported the classification performance lower than the baseline.

Results of the experiments to find the optimal hyperparameters of our method are as follows.
We first compared the classification performance given the sparsity 5\%, 10\%, 15\%, 20\%, 30\%, 40\% to find the optimal level of sparsity of the graph edges.
The level of sparsity versus classification accuracy was also tested with the GCN-based baseline method, which showed similar trend to the proposed method with slightly lower accuracy (Figure~\ref{fig:sparsity}).
The best performance was achieved with the sparsity 30\%, so we report the results with sparsity 30\% from here.
Results of the other hyperaparameter tuning experiments, including the regularization cofficient $\lambda$, dropout rate, learning rate, number of layers, and number of hidden units in each layers are summarized in the Table \ref{table:hyperparameter}.
Based on theses hyperparameter experiments, the final GIN model was implemented 5 layers deep with 64 hidden units in each layers.
Dropout was applied at the final linear layer with dropout rate of 0.5 during the training phase, and the regularization cofficient $\lambda$ of \eqref{eq:loss} was set to 0.05.

The last comparative study was on classification performance of different node embeddings.
It was found from the experiments that embedding the node feature as the centroid coordinate or the mean BOLD activity resulted in a significantly lower classification accuracy (Table \ref{table:embedding}).
To evaluate the latent space of the model trained with differently embedded node features, we visualized the latent space of the model with the t-SNE \citep{maaten2008visualizing}, and computed the silhouette score between the two classes \citep{rousseeuw1987silhouettes}.
The silhouette score represents how each subjects are well clustered to its class in the latent space.
The t-SNE visualization of the latent space in Fig.~\ref{fig:tsne} was found to be more linearly separable when trained with one-hot embedded node features, while other embedding methods showed highly entangled latent space.
The mean silhouette score of the test data across the ten-folds was 0.123 with the one-hot node features, while the BOLD mean, centroid coordinate, and the dimension matched centroid coordinate node features resulted in lower scores with 0.007, 0.014, 0.017, respectively.
 \par

\subsection{Saliency mapping}\label{sec:exp:saliency}
First, we demonstrate the robustness of the proposed saliency mapping method.
Experiment on the robustness of the proposed method showed average of 63.5\% and 65.5\% top region match on one fold aggregated saliency maps for female and male classes, respectively (Table~\ref{table:robustness}).
The robustness was higher for five fold aggregated result as expected, showing 92.5\% and 87.5\% top region match.
Significantly lower top region match with high standard deviation was found with the CAM-based saliency mapping method under same conditions.
This was especially notable for the saliency maps of the female class, which showed 46.5\% top region match on the one fold aggregated maps and 70.0\% top region match on the five fold aggregated maps.
\par

Plotted image and the full list of ROIs of the top 5-percentile salient regions from the proposed method are reported in the Figure~\ref{fig:saliency_proposed}, and the Table~\ref{table:saliency_proposed}.
The brain regions shown to be salient to the female class were the left prefrontal cortex (PFC), the right medial PFC, the right orbitofrontal cortex, the right cingulate cortex, the left frontal operculum, the left frontal eye field, the left temporal pole, the left temporal and parietal lobe regions, the bilateral visual cortex, and the bilateral somatomotor area.
The functional networks that these brain regions comprise include all seven networks from the Yeo 7 networks \citep{thomas2011organization}, which are the default mode network (DMN), the saliency/ventral attention network (SVN), the cognitive control network (CCN), the dorsal attention network (DAN), the limbic network (LN), the somatomotor network (SMN), and the visual network (VN). Among the seven networks, regions within the DMN was the most prominent taking up 30\% of the 20 regions, followed by the SMN (25\%), and the SVN (20\%).
Between the two hemispheres, salient regions were dominant in the left hemisphere (65\%) when compared to the right hemisphere (35\%).
\par

For the male class, salient regions were the left PFC, the right medial and lateral PFC, the left orbitofrontal cortex, the bilateral posterior cingulate cortex (PCC), the right precuneus, the bilateral cingulate cortex, the left temporal pole, the right temporal lobe region, the right intraparietal sulcus, the right visual cortex, and the bilateral somatomotor area.
The DMN was also predominant of all the functional networks as in the female class.
While ratio of the dominant networks in the male class showed a similar trend to the female class, the left hemisphere dominance was not present as in the female class (See pie charts of the Figure~\ref{fig:saliency_proposed}).
\par

Next, we explore the saliency mapping result from the CAM-based method \citep{arslan2018graph} and compare it with our method (Figure~\ref{fig:saliency_cam}, Table~\ref{table:saliency_cam}).
From the CAM-based methods, salient regions from both the female and the male class overlapped with our proposed method, including areas such as the PFC, the orbitofrontal cortex, the cingualte cortex, the PCC, the precuneus, and the temporal / parietal lobe regions.
The most notable difference was the absence of the regions from the SMN and the VN in both classes.
There were five functional networks that included the salient regions, the DMN, the SVN, the CCN, the DAN, and the LN.
The dominance ratio of these five functional networks were similar to that found in our proposed saliency mapping results.
In the male class, not only the regions from the SMN and the VN were missing, but also from the SVN, the DAN, and the LN.
The only salient regions in the male class were the left PFC, the right medial/lateral PFC, the left PCC, the left precuneus, the temporal lobe and the parietal lobe regions from the DMN and the CCN.
Hemisphere dominance showed a similar trend to the proposed method in that the female class clearly showed left hemisphere dominance (75\%), while the male class did not show any hemisphere dominance (50\%).

\section{Discussion}
In this study, we proposed a framework for analyzing the fMRI data with the GIN.
The framework suggests on first constructing the graph from the semantic region labels and the functional connectivity between them.
We train a GIN for classifying the subject phenotype based on the whole graph properties.
After training, we can classify the subject with the trained GIN, or visualize the regions related to the classification by backpropagating through the trained GIN.
An important theoretical basis that we found which underlie in this proposed method is that the GIN is not just a black-box operation that aggregates the graph structure with the MLP, but is actually a dual representation of a CNN on the graph space where the adjacency matrix is used
as a generalized shift operator.
\par

Classification of sex based on the rs-fMRI data resulted in the accuracy, precision, and recall of 84.61\%, 86.19\%, and 86.81\%, respectively.
The performance of the classifier is at least comparable, if not outperforming, to other recent methods for classifying sex based on the rs-fMRI data of the HCP dataset \citep{arslan2018graph, zhang2018functional, weis2019sex} (Table \ref{table:comparison}).
Through the comparative studies, we have shown the validity of our proposed method that it can accurately classify the sex of the subjects with the rs-fMRI data.
When training the GNN, adding the Infomax regularization had improved the classification performance of the GIN (Table \ref{table:hyperparameter}).
We have not gone through extensive experiment regarding the role of the Infomax regularization, but suggest to add it when training the neural network based on the results of our experiment.
One interesting finding in our comparative experiments was that embedding the node feature as vectors of centroid coordinate or mean BOLD activity results in a significantly lower classification performance (Table \ref{table:embedding}).
We expect that this comes from the linear dependence of the node features when embedded with centroid coordinate or mean BOLD activity.
Further discussion regarding this topic is covered in the Appendix.
\par

After fully training the GIN for the sex classification task, we could map the salient regions related to the classification by the saliency mapping method.
From the saliency mapping result, we could find that the regions within the DMN takes the most prominent role in classifying both the female and the male subjects.
Importance of the DMN in the sex classification based on rs-fMRI data has been consistently reported \citep{zhang2018functional,weis2019sex}.
In the study by Zhang and colleagues \citep{zhang2018functional}, there were seven features involving the DMN of the top twenty important regions (35\%) for sex classification, which is similar to our result (30\% for the female class and 35\% for the male class).
This importance of the DMN for the sex classification task is known to be related to the difference of the DMN functional connectivity between the two sexes during the resting-state \citep{mak2017default}.
Considering the difference of the DMN between the two sexes, it has been found consistently, and also from the meta-analysis, that the female individuals show stronger functional connectivity of the DMN compared to the males \citep{bluhm2008default, biswal2010toward, allen2011baseline, mak2017default, zhang2018functional}.
CAM-based saliency mapping method also reflected this difference in the DMN between the two sexes and has shown predominance of the DMN in the saliency map, which is replicative of the original CAM-based saliency mapping study by \citep{arslan2018graph}.
These findings suggest the validity of our saliency mapping method that it corresponds to the previous neuroimaging evidences regarding the importance of the DMN in sex classification.
\par

Hemisphere related sex differences are also previously reported \citep{tian2011hemisphere,hjelmervik2014resting}.
The studies indicate that female subjects show higher functional connectivity in the left hemisphere, and male subjects in the right hemisphere \citep{tian2011hemisphere}.
This difference in hemisphere dominance has shown the same trend in our experiment.
In the female class, the the salient regions in the left hemisphere outnumbered the salient regions in the right hemisphere (left 65\% vs. right 35\%), whereas the male class resulted in the right hemisphere lateralized saliency mapping result (left 45\% vs. right 55\%).
The left hemisphere dominance of the female class was also found from the CAM-based saliency mapping results (left 75\% vs. right 25\%), but was not apparent in the male class (left 50\% vs. right 50\%).
We interpret that the hemisphere related sex differences found in our saliency mapping result further supports the validity of our method.
\par

Given the validity of the proposed saliency mapping method, the novel advantages of our method is highlighted by comparing it with the results from the CAM-based method.
We find that the two major advantages over the CAM-based method are the robustness and the mapping sensitivity.
The advantage in robustness is suggested from the experiment result that our proposed method captures more consistent top salient regions than the CAM-based method even with small number of subjects (Table \ref{table:robustness}).
The other advantage, mapping sensitivity, is implied in the saliency mapping results.
Mapping results from our method revealed the involvement of the regions within the SMN and the VN, while the CAM-based method was not able to identify them (Figure \ref{fig:saliency_proposed} and \ref{fig:saliency_cam}).
There are some previous studies noting that there exist difference between the two sexes in terms of the functional connectivity within the SMN and the VN \citep{allen2011baseline, zhang2018functional, xu2015gender}.
However, the evidences supporting this difference in the SMN and the VN are not as prominent and well established as the difference in the DMN between the two sexes.
It can be said that another supportive evidence of the difference of the SMN and the VN between the two sexes is added to the functional neuroimaging field by the proposed saliency mapping method, which would had not been identified by the CAM-based method.
Based on this mapping sensitivity, applying the proposed method other types of classification tasks or to other subject groups is expected to provide new interesting findings to the neuroscientific field.
To sum up, the proposed GIN based rs-fMRI analysis framework achieves state-of-the-art classification performance while providing a robust and sensitive saliency map which can be interpreted to add new insights to the field of functional neuroimaging.

There are some limitations and caveats that needs to be discussed.
First, the demographics that can affect the analysis have not been considered or controlled thoroughly.
It is well known that the resting-state network can be affected by the age, handedness, fluid intelligence, and other subject charactersitics.
The results are expected to have stronger explainability by taking the demographics of the subjects into account in the analysis.
Second, the cutoff threshold for determining the salient region was heuristically set.
We have set the regions with the top 5 percentile values as salient, but the method would have even more validity if the salient regions were determined in a more data-driven way, as in the classical methods perform statistical testing to determine the significance of each voxels.
We have not gone through extensive study on the topic of determining the significant regions from the saliency map, but is worth further studies and discussion. \par

Still, we insist that analyzing the fMRI data based on the GIN has shown its theoretical and experimental validity in this study.
We believe that the GIN based analysis method offers a potential advancement in the area, by opening a way to exploit the capability of the GIN to learn highly non-linear mappings.
Some interesting topics related to this work can be considered.
Theoretically, exploring the operations beyond the two-tab convolution filter by GIN
can potentially provide better performance than the existing GIN.
Neuroscientifically, extension of the method to clinical data interpretation or to the multi-class graph classification problem can be interesting topics in the future.
With enough data assured, the proposed method is expected to help reveal new findings from the functional networks of the brain.

\appendix
\section*{Appendix}
\subsection*{Expressive explanation of the GIN operation}

This section is devoted to explaining the GIN operation \eqref{eq:gin} in a more expressive manner.
For that purpose, we consider a small graph with four nodes ($N=4$) and four edges (Figure \ref{fig:smallgraph_example}).
An example one-hot input node feature matrix $\Xb$ and the adjacency matrix $\Ab$ are defined as following (Figure \ref{fig:smallgraph_example}. A),
\begin{align*}
\Xb^{(0)} &= \begin{bmatrix} 1 & 0 & 0 & 0 \\ 0 & 1 &  0 & 0 \\ 0 & 0 & 1 & 0 \\ 0 & 0 & 0 & 1\end{bmatrix} \\
\Ab &= \begin{bmatrix} 0 & 1 & 1 & 0 \\ 1 & 0 &  1 & 1 \\ 1 & 1 & 1 & 0 \\ 0 & 1 & 0 & 0\end{bmatrix}.
\end{align*}

Next, each node feature is multiplied by $1+\epsilon^{(0)}$, and the neighboring node features are summed, as in the brackets of the \eqref{eq:gin} (Figure \ref{fig:smallgraph_example}. B). In this example, we set the learnable parameter $\epsilon^{(0)}=0.1$ to obtain the aggregated feature matrix
$$
\Rb^{(0)} = \begin{bmatrix} 1.1 & 1 & 1 & 0 \\ 1 & 1.1 &  1 & 0 \\ 1 & 1 & 1.1 & 1 \\ 0 & 0 & 0 & 1.1 \end{bmatrix}.
$$
The aggregated feature matrix $\Rb^{(0)}$ is mapped through the MLP and then the ReLU nonlinearity (Figure \ref{fig:smallgraph_example}. C, D).

Here we set the example MLP weight matrix $\Wb^{(0)}$ as
$$
\Wb^{(0)} = \begin{bmatrix} 0.1 & -0.2 & -0.3 & 0.4 \\ -0.1 & 0.2 & -0.3 & 0.4 \\ 0.4 & 0.3 & 0.2 & -0.1 \\ -0.4 & 0.3 & 0.2 & -0.1 \end{bmatrix}.
$$
to obtain the next layer feature matrix $\Xb^{(1)}$ as
\begin{align*}
\Xb^{(1)} &= \sigma \Bigl( \begin{bmatrix} 0.11 & -0.3 & -0.2 & 0 \\ -0.1 & -0.33 &  0.2 & 0 \\ 0.4 & 0.2 & 0.33 & -0.1 \\ 0 & 0 & 0 & -0.11 \end{bmatrix} \Bigr) \\
&= \begin{bmatrix} 0.11 & 0 & 0 & 0 \\ 0 & 0 &  0.2 & 0 \\ 0.4 & 0.2 & 0.33 & 0 \\ 0 & 0 & 0 & 0 \end{bmatrix}.
\end{align*}

Like the above example process, same operations are applied to the mapped node feature of the $k$-th layer $\Xb^{(k)}$ in each layers of the GIN \eqref{eq:gin}.

\subsection*{Linear dependence of the node features reduces the discriminative power of the GIN}
We demonstrate the importance of embedding the node as a one-hot vector, based on the idea of \citep{xu2018powerful}.
In terms of graph classification tasks, \citep{xu2018powerful} has shown that the GIN can be as discriminative as the Weisfeiler-Lehman (WL) test.
WL test is a test for solving the graph isomorphism problem, where graph isomorphism means whether two seperate graphs are topologically identical.
Conditions for achieving this maximum discriminative power is also provided by \citep{xu2018powerful}, which are that if the $\mathtt{AGGREGATE}$, $\mathtt{COMBINE}$, and $\mathtt{READOUT}$ mappings of the GNN are injective, then the GNN is as powerful as the WL test.
One issue here is that $\mathtt{AGGREGATE}$ is usually implemented as sum operation, so
 if input node feature vectors are not linearly independent, it leads to a non-injective mapping.

For example, assume that two distinct node features are embedded into two vectors $\pb_{1}^{(0)}$ and $\pb_{2}^{(0)}$, respectively.
If $\pb_{1}^{(0)}$ and $\pb_{2}^{(0)}$ are linearly dependent, then there exist positive integers $a$ and $b$ such that
\begin{align} \label{eq:lineardependence}
\pb_{1}^{(0)} = \frac{a}{b} \pb_{2}^{(0)}
\end{align}
Consider one simple graph $G_{1}$ that comprise $b+1$ vertices with all features  as $\pb_{1}^{(0)}$ where there exists edges between the first node and the others, that is
\begin{align*}
  \Xb^{(0)} &= {\underbrace{\begin{bmatrix} \pb_{1}^{(0)} &\pb_{1}^{(0)} &\cdots & \pb_{1}^{(0)} \end{bmatrix}}_{b+1}}^{\top}, \\
  \Ab &= \begin{bmatrix} 0 & 1 & \cdots & 1 & 1 \\ 1 & 0 & \cdots &  0 & 0 \\ \vdots & \vdots & \ddots & \vdots & \vdots \\
  1 & 0 & \cdots & 0 & 0 \\
  1 & 0 & \cdots & 0 & 0 \end{bmatrix}
\end{align*}
Then the node feature vectors of the next layer from \eqref{eq:gin2} is
\begin{align}\label{eq:g1feat}
\rb_{1}^{(1)} &= c^{(1)}\pb_{1}^{(0)} + b\pb_{1}^{(0)} \\
\rb_{v}^{(1)} &= c^{(1)}\pb_{1}^{(0)} + \pb_{1}^{(0)}, \quad v=2,3,\cdots,b+1 \notag
\end{align}

Now, consider another simple graph $G_{2}$ that comprise $a+1$ vertices with a feature $\pb_{1}^{(0)}$, and the other features as  $\pb_{2}^{(0)}$ where there exists edges between the first node and the others, that is
\begin{align*}
  \Xb^{(0)} &= {\underbrace{\begin{bmatrix} \pb_{1}^{(0)} &\pb_{2}^{(0)} &\cdots & \pb_{2}^{(0)} \end{bmatrix}}_{a+1}}^{\top}, \\
  \Ab &= \begin{bmatrix} 0 & 1 & \cdots & 1 & 1 \\ 1 & 0 & \cdots &  0 & 0 \\ \vdots & \vdots & \ddots & \vdots & \vdots \\
  1 & 0 & \cdots & 0 & 0 \\
  1 & 0 & \cdots & 0 & 0 \end{bmatrix}
\end{align*}
Then the node feature vectors of the next layer from \eqref{eq:gin2} is
\begin{align}\label{eq:g2feat}
\rb_{1}^{(1)} &= c^{(1)}\pb_{1}^{(0)} + a\pb_{2}^{(0)} \\
\rb_{v}^{(1)} &= c^{(1)}\pb_{2}^{(0)} + \pb_{2}^{(0)}, \quad v=2,3,\cdots,a+1 \notag
\end{align}

We can now see that the first layer node embeddings of the $G_{1}$ \eqref{eq:g1feat} and the $G_{2}$ \eqref{eq:g2feat} are identical given \eqref{eq:lineardependence},
\begin{align*}
\rb_{1}^{(1)} &= c^{(1)}\pb_{1}^{(0)} + b\pb_{1}^{(0)} \\
&= c^{(1)}\pb_{1}^{(0)} + a\pb_{2}^{(0)}.
\end{align*}
In this case, regardless of the MLP,
 the embedding of the first node of the $G_{1}$ and the $G_{2}$ are not discriminative with the GIN \eqref{eq:gin}.

Thus, it can be said that it is more practical to make the set of input feature vectors linearly independent to each other.
By embedding each separate ROIs into a one-hot vector encoding, it can be ensured that the input features are orthogonal, needless to say linearly independent, to each other. Moreover, one-hot vector encoding leads to a more interpretable Grad-CAM saliency map as in \eqref{eq:saliency}.

\section*{Acknowledgments}
The authors would like to thank Sangmin Lee for his useful comments.


\section*{Author Contributions}

BK designed/conducted the experiments, interpreted the neuroscientific findings, and wrote the manuscript. JY supervised the experiments, deduced the theoretical findings, and wrote the manuscript.
%

\section*{Conflict of Interest Statement}

The authors declare that the research was conducted in the absence of any commercial or financial relationships that could be construed as a potential conflict of interest.

\section*{Data Availability Statement}
The code and datasets for this study can be found in \url{https://github.com/egyptdj/graph-neural-mapping}.

\bibliographystyle{frontiersinSCNS_ENG_HUMS} 



\begin{landscape}
\pagebreak
\section*{Tables and Figures}
\begin{table}[ht]
\renewcommand{\arraystretch}{1.3}
\caption{Comparison of various methods for sex classification with the HCP dataset}
\label{table:comparison}
\centering
\begin{tabular}{c | c c c | c c c c c}
\hline
\hline
 Model         & Accruacy (\%) & Precision (\%) & Recall (\%) & Subjects & Parcellation & Validation & Author   & Year \\
\hline
GIN + Infomax & \textbf{84.61 $\pm$ 2.9} & 86.19 $\pm$ 3.3 & 86.81 $\pm$ 4.9 & 942 & Schaefer 400 & 10-fold & Ours  & 2020 \\
GIN           & 84.41 $\pm$ 2.8 & 85.39 $\pm$ 2.6 & 87.60 $\pm$ 7.5 & 942 & Schaefer 400 & 10-fold & Ours & 2020 \\
SVM-RBF & 68.7 $\pm$ 2.6 & - & - & 434 &  Schaefer 400 + Fan 39 & 10-fold & Weis et al.  & 2019 \\
SVM-RBF & 64.3 $\pm$ 2.6 & - & - & 310 &  Schaefer 400 + Fan 39 & Separate & Weis et al. & 2019\\
GCN* (baseline) & 83.98 $\pm$ 3.2 & 84.59 $\pm$ 3.1 & 87.78 $\pm$ 6.4 & 942 &  Schaefer 400 & 10-fold & Arslan et al.  & 2018 \\
PLS & 79.9 $\pm$ 0.9 & - & - & 820 & Dosenbach 160 & 10-fold & Zhang et al.  & 2018\\
\hline
\hline
\multicolumn{3}{l}{* Re-implemented to test for the HCP dataset.}
\end{tabular}
\end{table}

\begin{table}[ht]
\renewcommand{\arraystretch}{1.3}
\caption{Hyperparameter tuning experiments}
\label{table:hyperparameter}
\centering
\begin{tabular}{c|c c c c c | c c c}
\hline
\hline
Model      & $\lambda$   & Dropout & Learning rate   & Layers & Hidden units & Accuracy (\%) & Precision (\%) & Recall (\%) \\
\hline
GCN (Baseline) & None & 0.5 (2,4,5 layer)  & 0.005 & 5 & 32/32/64/64/128 & 83.98 $\pm$ 3.2 & 84.59 $\pm$ 3.1 & 87.78 $\pm$ 6.4\\
\hline
GIN+Infomax           & 0.05 & 0.5   & 0.005 & 5  & 64 & \textbf{84.61 $\pm$ 2.9} & 86.19 $\pm$ 3.3 & 86.81 $\pm$ 4.9\\
GIN                   & 0.0  &  -    & -     & -  & -  & 84.41 $\pm$ 2.8 & 85.39 $\pm$ 2.6 & 87.60 $\pm$ 7.5\\
\hline
-                     & 0.01 &  -    & -     & -  & -  & 84.08 $\pm$ 2.2 & 86.72 $\pm$ 4.4 & 85.31 $\pm$ 5.5\\
-                     & 0.1  &  -    & -     & -  & -  & 84.51 $\pm$ 2.1 & 86.85 $\pm$ 4.5 & 86.06 $\pm$ 5.5\\
-                     & -    & 0.0   & -     & -  & -  & 83.99 $\pm$ 3.4 & 85.78 $\pm$ 4.4 & 86.26 $\pm$ 6.1\\
-                     & -    &  -    & 0.01  & -  & -  & 83.13 $\pm$ 3.4 & 85.89 $\pm$ 3.4 & 84.01 $\pm$ 5.2\\
-                     & -    &  -    & 0.001 & -  & -  & 81.54 $\pm$ 3.3 & 85.45 $\pm$ 3.4 & 81.37 $\pm$ 7.3\\
-                     & -    &  -    & -     & 4  & -  & 83.11 $\pm$ 3.2 & 84.62 $\pm$ 2.8 & 85.70 $\pm$ 4.2\\
-                     & -    &  -    & -     & -  & 32 & 83.13 $\pm$ 3.4 & 85.20 $\pm$ 4.3 & 85.14 $\pm$ 5.5\\
\hline
\hline
\end{tabular}
\end{table}
\end{landscape}

\begin{table}[ht]
\renewcommand{\arraystretch}{1.3}
\caption{Comparison of different node feature embeddings}
\label{table:embedding}
\begin{center}
  \begin{tabular}{c | c c c }
  \hline
  \hline
  Node feature  & Accuracy (\%)  & Precision (\%) & Recall (\%) \\
  \hline
  One-hot     & 84.61 $\pm$ 2.9  & 86.19 $\pm$ 3.3  & 86.81 $\pm$ 4.9 \\
  BOLD mean   & 67.73 $\pm$ 2.9  & 69.90 $\pm$ 4.1  & 76.46 $\pm$ 8.3 \\
  Coordinate  & 72.19 $\pm$ 4.4  & 76.06 $\pm$ 6.8  & 75.88 $\pm$ 7.2 \\
  Coordinate* & 70.90 $\pm$ 4.1  & 72.94 $\pm$ 4.9  & 78.33 $\pm$ 8.6 \\
  \hline
  \hline
  \multicolumn{4}{l}{* Dimension matched to one-hot}

  \end{tabular}
\end{center}
\end{table}

\begin{table}[ht]
\renewcommand{\arraystretch}{1.3}
\caption{Robustness of the saliency mapping methods}
\label{table:robustness}
\begin{center}
  \begin{tabular}{c | c c | c c }
  \hline
  \hline
  Method  & \multicolumn{2}{c|}{Proposed} & \multicolumn{2}{c}{CAM}  \\
  \hline
          & One-fold & Five-folds & One-fold & Five-folds \\
  \hline
  Female  & 63.5$\pm$6.7 \% & 92.5$\pm$2.5 \% & 46.5$\pm$16.9 \% & 70.0$\pm$5.0 \% \\
  Male    & 65.5$\pm$5.2 \% & 87.5$\pm$7.5 \% & 62.0$\pm$25.4 \% & 92.5$\pm$2.5 \% \\
  \hline
  \hline
  \end{tabular}
\end{center}
\end{table}

\begin{landscape}
\begin{table}[ht]
\renewcommand{\arraystretch}{1.3}
\caption{Top 5-percentile salient regions identified by the proposed method for the female and the male class}
\label{table:saliency_proposed}
\begin{center}
  \begin{tabular}{c c c c c c c | c c c c c c c}
  \hline
  \hline
  \multicolumn{7}{c|}{{\em Female}} &   \multicolumn{7}{c}{{\em Male}} \\
  \hline
  Side   & Region & Network    & R & A & S  & Value  &   Side   & Region & Network    & R & A & S  & Value  \\
  \hline
  L. & Somatomotor area & SMN & -8 & -42 & 70 & 1.000 &   R. & Medial PFC & DMN & 10 & 66 & 0 & 1.000\\

  R. & Somatomotor area & SMN & 64 & -34 & 10 & 0.968 &   L. & Somatomotor area & SMN & -48 & -12 & 14 & 0.986\\

  L. & Visual cortex & VN & -18 & -64 & 6 & 0.951 &   R. & PCC & DMN & 8 & -44 & 20 & 0.985\\

  R. & Medial PFC & DMN & 8 & 54 & 12 & 0.931 &   L. & Somatomotor area & SMN & -58 & -36 & 16 & 0.976\\

  R. & Visual cortex & VN & 4 & -80 & 24 & 0.909 &   R. & Cingulate cortex & SVN & 6 & 10 & 58 & 0.973\\

  R. & Orbitofrontal cortex & LN & 20 & 42 & -18 & 0.887 &   L. & PCC & DMN & -4 & -54 & 20 & 0.960\\

  L. & PFC & DMN & -6 & 34 & 20 & 0.863  &   R. & Temporal lobe & DMN & 48 & 16 & -20 & 0.951\\

  L. & PFC & DMN & -22 & 20 & 52 & 0.835 &   L. & Cingulate cortex & SVN & -6 & -48 & 56 & 0.949\\

  L. & PFC & DMN & -36 & 36 & -12 & 0.835 &   R. & Somatomotor area & SMN & 12 & -18 & 42 & 0.935\\

  R. & Somatomotor area & SMN & 6 & -22 & 72 & 0.832 &   L. & PFC & DMN & -14 & 58 & 30 & 0.932\\

  L. & Temporal lobe & DMN & -40 & -78 & 30 & 0.821 &  R. & Cingulate cortex & SVN & 16 & 6 & 70 & 0.908\\

  L. & Parietal lobe & CCN & -44 & -42 & 46 & 0.816 &  R. & Visual cortex & VN & 24 & -74 & -10 & 0.899\\

  L. & Somatomotor area & SMN & -52 & -6 & 44 & 0.814 &  R. & Lateral PFC & CCN & 44 & 18 & 44 & 0.881\\

  L. & Frontal operculum & SVN & -52 & 8 & 14 & 0.806 &  L. & Orbitofrontal cortex & LN & -16 & 64 & -8 & 0.881\\

  L. & Frontal operculum & SVN & -44 & 6 & -16 & 0.806 &  L. & PFC & DMN & -18 & 36 & 48 & 0.854\\

  L. & Temporal pole & LN & -54 & -22 & -30 & 0.804 &  R. & Intraparietal sulcus & DAN & 8 & -72 & 52 & 0.843\\

  L. & PFC & DMN & -8 & 42 & 52 & 0.799 &  L. & Temporal pole & LN & -26 & -10 & -32 & 0.835\\

  R. & Somatomotor area & SMN & 40 & -20 & 4 & 0.784 &  R. & Visual cortex & VN & 36 & -88 & 2 & 0.835\\

  R. & Cingulate cortex & SVN & 6 & -2 & 66 & 0.781 &  L. & PCC & DMN & -6 & -40 & 24 & 0.834\\

  L. & Frontal eye field & DAN & -26 & 0 & 56 & 0.777 &  R. & Precuneus & CCN & 14 & -72 & 40 & 0.834\\
  \hline
  \hline
  \end{tabular}
\end{center}
\end{table}

\begin{table}[ht]
\renewcommand{\arraystretch}{1.3}
\caption{Top 5-percentile salient regions identified by the CAM-based method for the female and the male class}
\label{table:saliency_cam}
\begin{center}
  \begin{tabular}{c c c c c c c | c c c c c c c}
  \hline
  \hline
  \multicolumn{7}{c|}{{\em Female}} &   \multicolumn{7}{c}{{\em Male}} \\
  \hline
  Side   & Region & Network    & R & A & S  & Value  & Side   & Region & Network    & R & A & S  & Value      \\
  \hline
  L. & PFC & DMN & -30 & 14 & 58 & 1.000 &   L. & PCC & DMN & -8 & -52 & 10 & 1.000\\

  L. & PFC & DMN & -8 & 42 & 52 & 0.965 &   R. & Parietal lobe & DMN & 54 & -46 & 20 & 0.983\\

  L. & PFC & DMN & -42 & 8 & 48 & 0.964 &   L. & PCC & DMN & -14 & -60 & 18 & 0.970\\

  L. & PFC & DMN & -22 & 20 & 52 & 0.944 &   R. & Parietal lobe & DMN & 48 & -64 & 22 & 0.961\\

  L. & Frontal operculum & SVN & -52 & 8 & 14 & 0.884 &   R. & Medial PFC & DMN & 26 & 34 & 38 & 0.953\\

  R. & Cingulate cortex & SVN & 6 & -2 & 66 & 0.880 &   R. & Parietal lobe & DMN & 56 & -46 & 32 & 0.941\\

  L. & Orbitofrontal cortex & LN & -12 & 24 & -20 & 0.870 &   R. & Medial PFC & CCN & 8 & 34 & 24 & 0.922\\

  L. & PFC & DMN & -22 & 50 & 32 & 0.852 &   R. & Parietal lobe & DMN & 54 & -54 & 26 & 0.920\\

  L. & Parietal lobe & CCN & -58 & -42 & 46 & 0.835 &   R. & Temporal lobe & DMN & 48 & 16 & -20 & 0.914\\

  L. & Parietal lobe & SVN & -62 & -24 & 32 & 0.819 &   L. & Temporal lobe & DMN & -60 & -36 & -18 & 0.865\\

  L. & Frontal operculum & SVN & -50 & 2 & 4 & 0.812 &   L. & Temporal lobe & DMN & -62 & -18 & -20 & 0.865\\

  L. & Temporal pole & LN & -24 & 6 & -40 & 0.804 &   L. & Temporal lobe & DMN & -60 & -34 & -4 & 0.858\\

  L. & Intraparietal sulcus & DAN & -14 & -50 & 72 & 0.798 &   L. & Temporal lobe & DMN & -52 & -22 & -6 & 0.853\\

  L. & Parietal lobe & CCN & -34 & -62 & 48 & 0.796 &  R. & Lateral PFC & CCN & 42 & 6 & 50 & 0.844\\

  R. & Frontal operculum & SVN & 54 & 12 & 12 & 0.784 &   R. & Temporal lobe & DMN & 50 & 8 & -32 & 0.816\\

  R. & Orbitofrontal cortex & LN & 14 & 24 & -20 & 0.780 &   L. & Temporal lobe & DMN & -58 & -48 & 16 & 0.811\\

  L. & Parietal lobe & CCN & -44 & -42 & 46 & 0.777 &   L. & PFC & DMN & -6 & 10 & 64 & 0.796\\

  R. & Medial PFC & DMN & 18 & 64 & 16 & 0.770 &   R. & Medial PFC & CCN & 4 & 28 & 48 & 0.789\\

  L. & PFC & DMN & -36 & 36 & -12 & 0.770 &   L. & Precuneus & CCN & -10 & -78 & 46 & 0.777\\

  R. & Medial PFC & DMN & 8 & 54 & 12 & 0.769 &   L. & PFC & DMN & -12 & 24 & 60 & 0.772\\

  \hline
  \hline
  \end{tabular}
\end{center}
\end{table}
\end{landscape}

\pagebreak

\begin{figure}[ht]
\begin{center}
\includegraphics[width=0.9\textwidth]{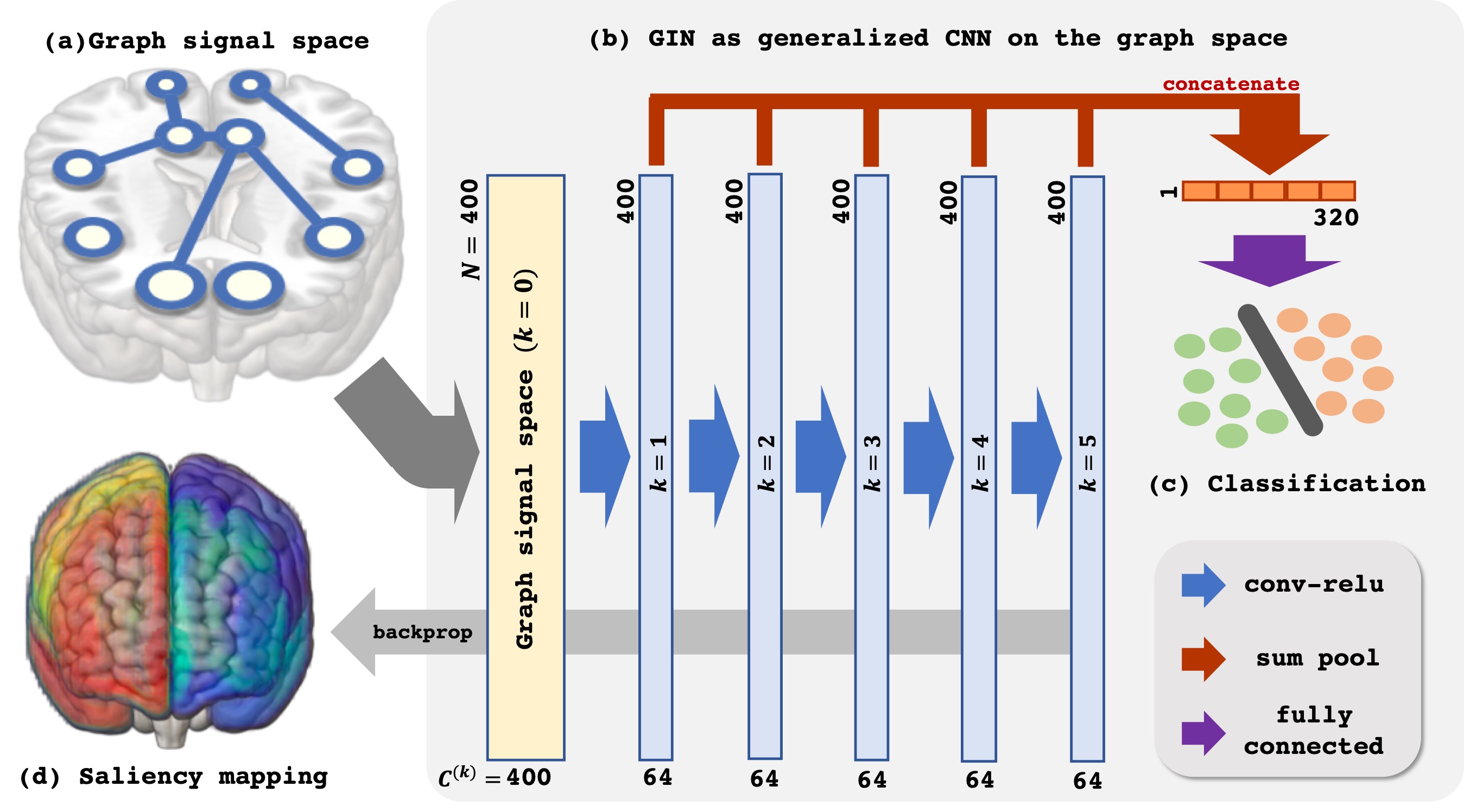}
\end{center}
\caption{Schematic illustration of the Graph Isomorphism Network based resting-state fMRI analysis.}
\label{fig:scheme}
\end{figure}

\begin{figure}[ht]
\begin{center}
\includegraphics[width=0.5\textwidth]{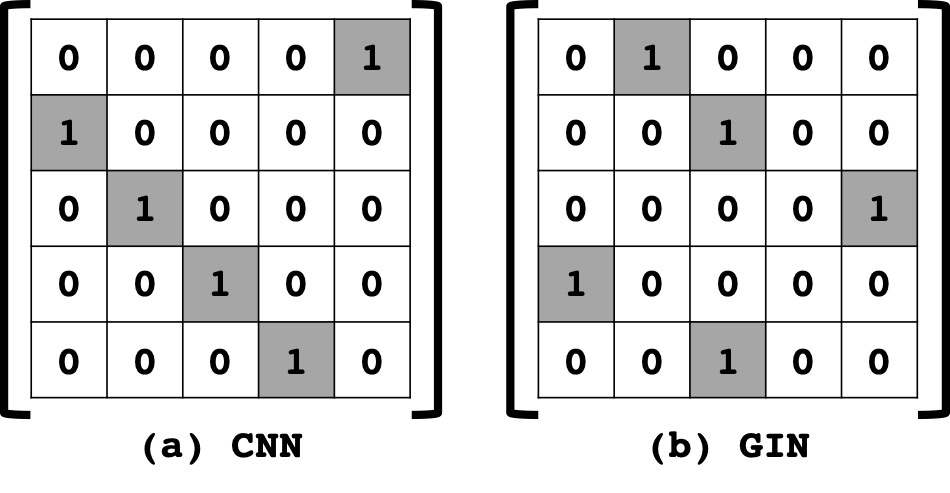}
\end{center}
\caption{Comparison of shift operation in (a) classical CNN, and (b) an example of GIN. In the graph space, the adjacency matrix is defined as shift operation.}
\label{fig:shift}
\end{figure}

\begin{figure}[ht]
\begin{center}
\includegraphics[width=1.0\textwidth,trim=4 4 4 4,clip]{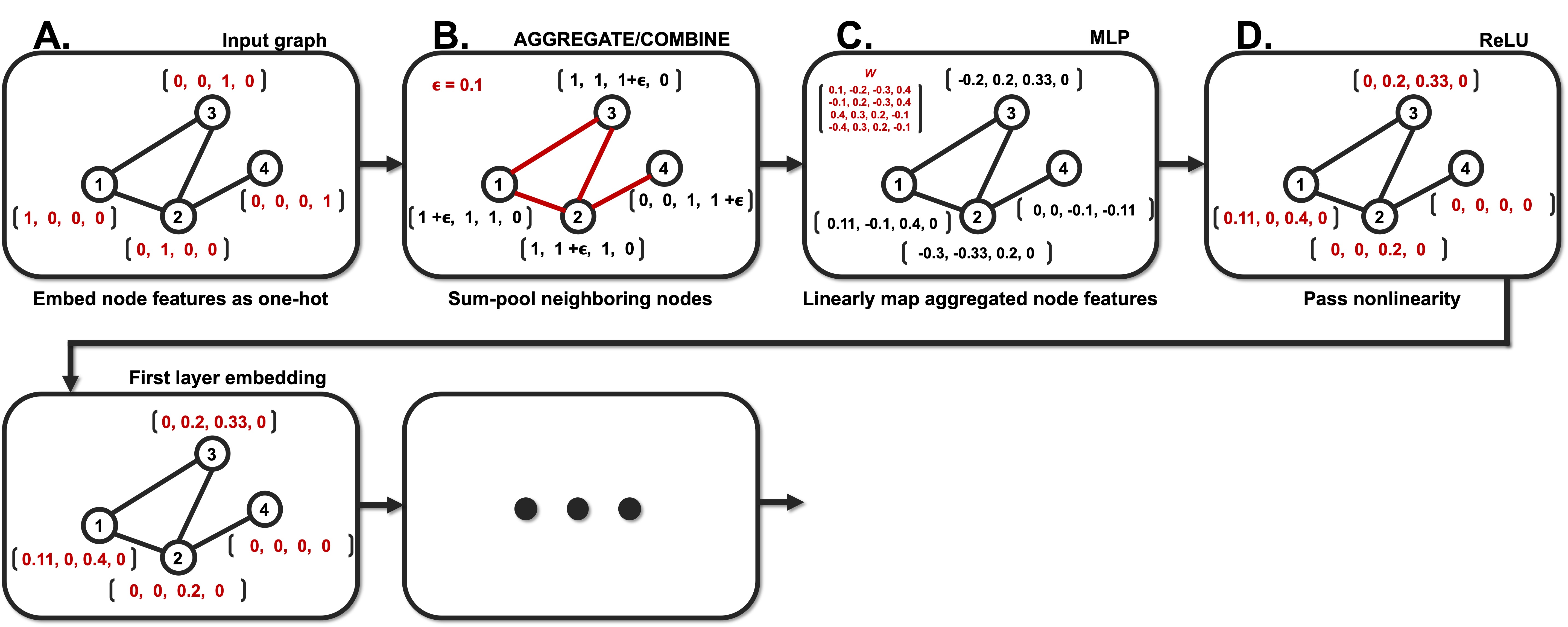}
\end{center}
\caption{Example of the GIN operation with a small graph ($N=4$).}
\label{fig:smallgraph_example}
\end{figure}

\begin{figure}[ht]
\begin{center}
\includegraphics[width=0.7\textwidth,trim=4 4 4 4,clip]{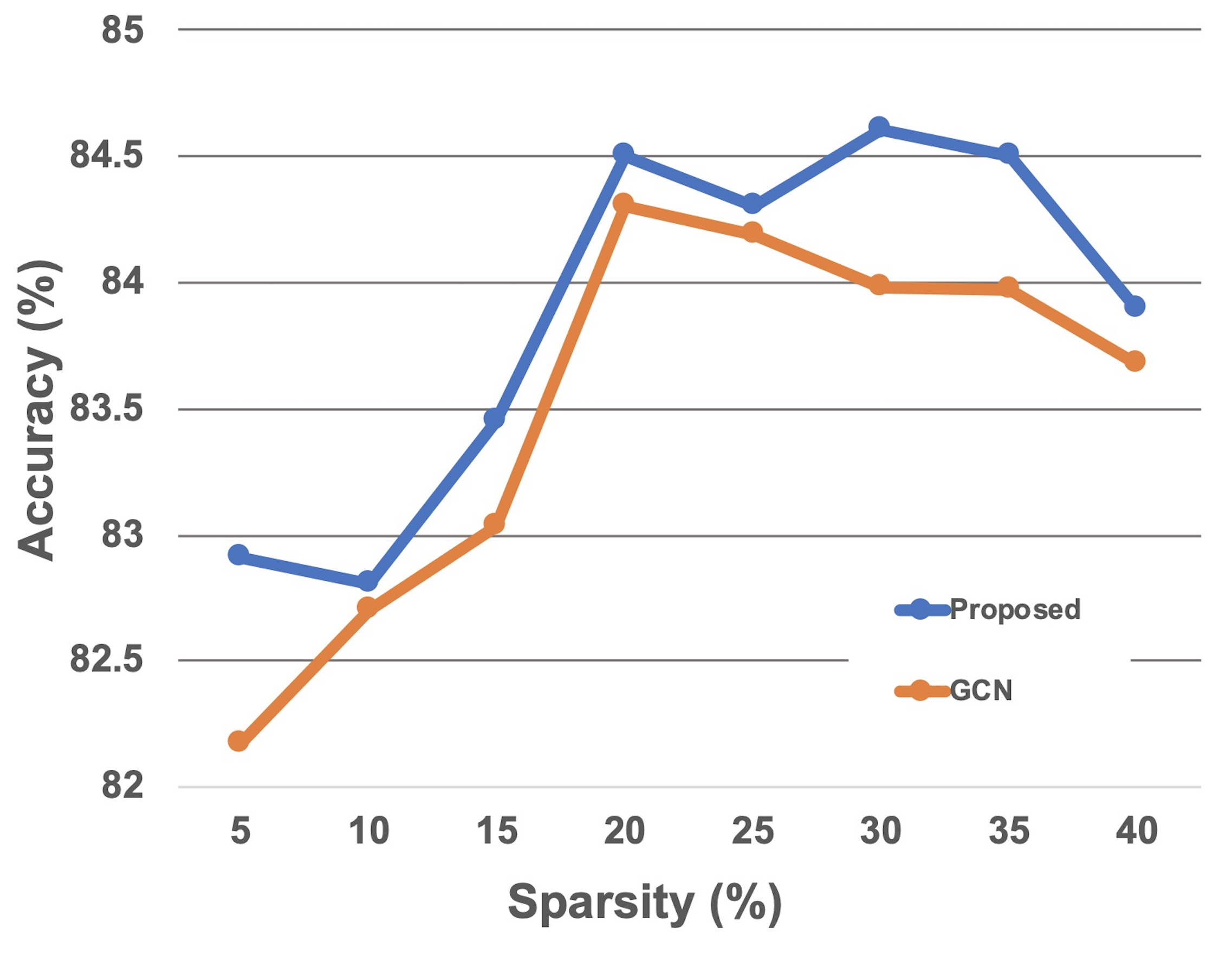}
\end{center}
\caption{Classification accuracy with respect to the edge sparsity.}
\label{fig:sparsity}
\end{figure}

\begin{figure}[ht]
\begin{center}
\includegraphics[width=0.8\textwidth,trim=4 4 4 4,clip]{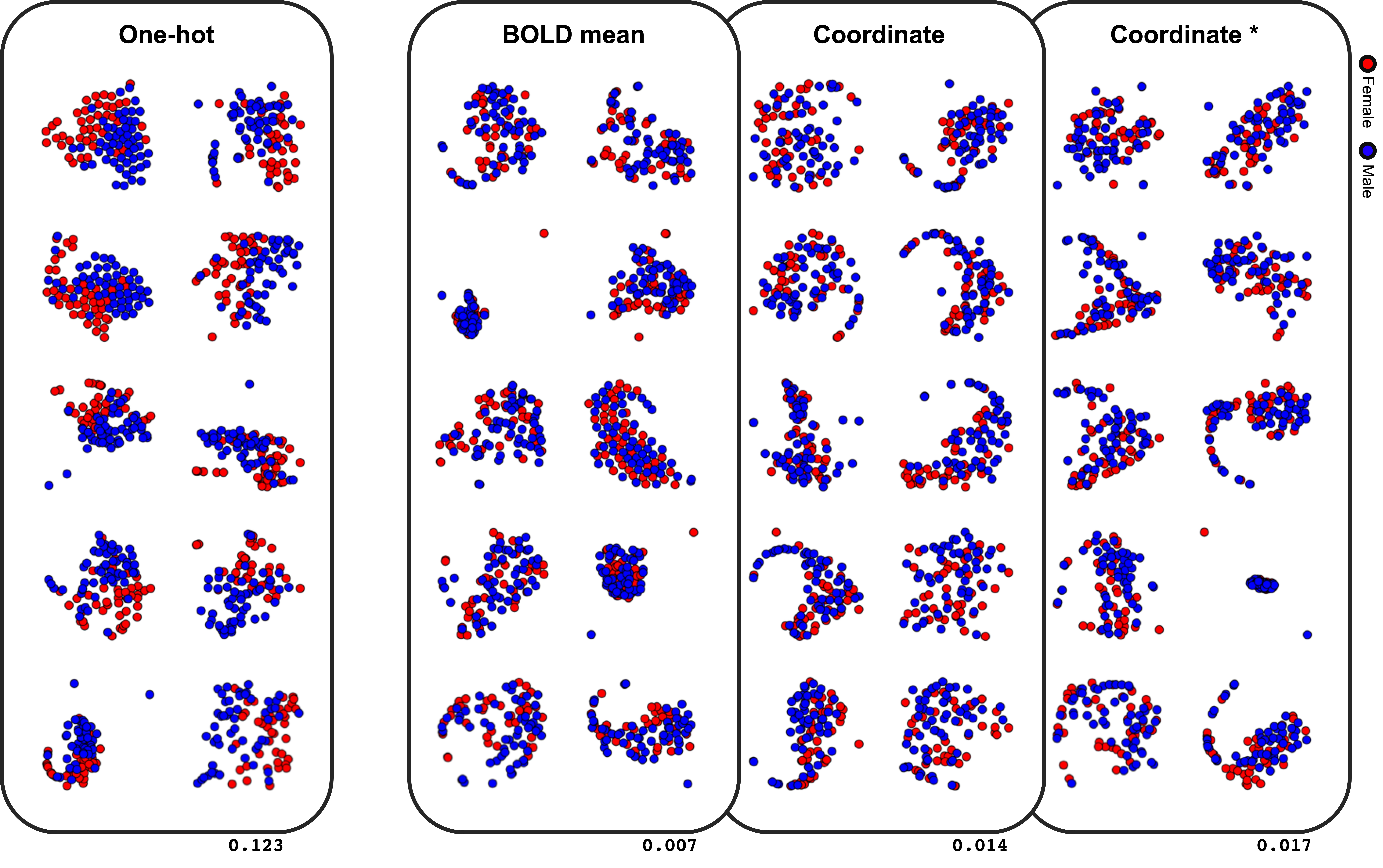}
\end{center}
\caption{Visualization of the latent space with t-SNE. Values on the lower right indicate the mean silhouette score of each embedding methods. Results of the ten fold cross-validation are plotted in separate spaces with perplexity 50. Asterisk indicate that the dimension is matched to the one-hot embedding.}
\label{fig:tsne}
\end{figure}

\begin{landscape}
\begin{figure}[ht]
\begin{center}
\includegraphics[width=1.35\textwidth,trim=4 4 4 4,clip]{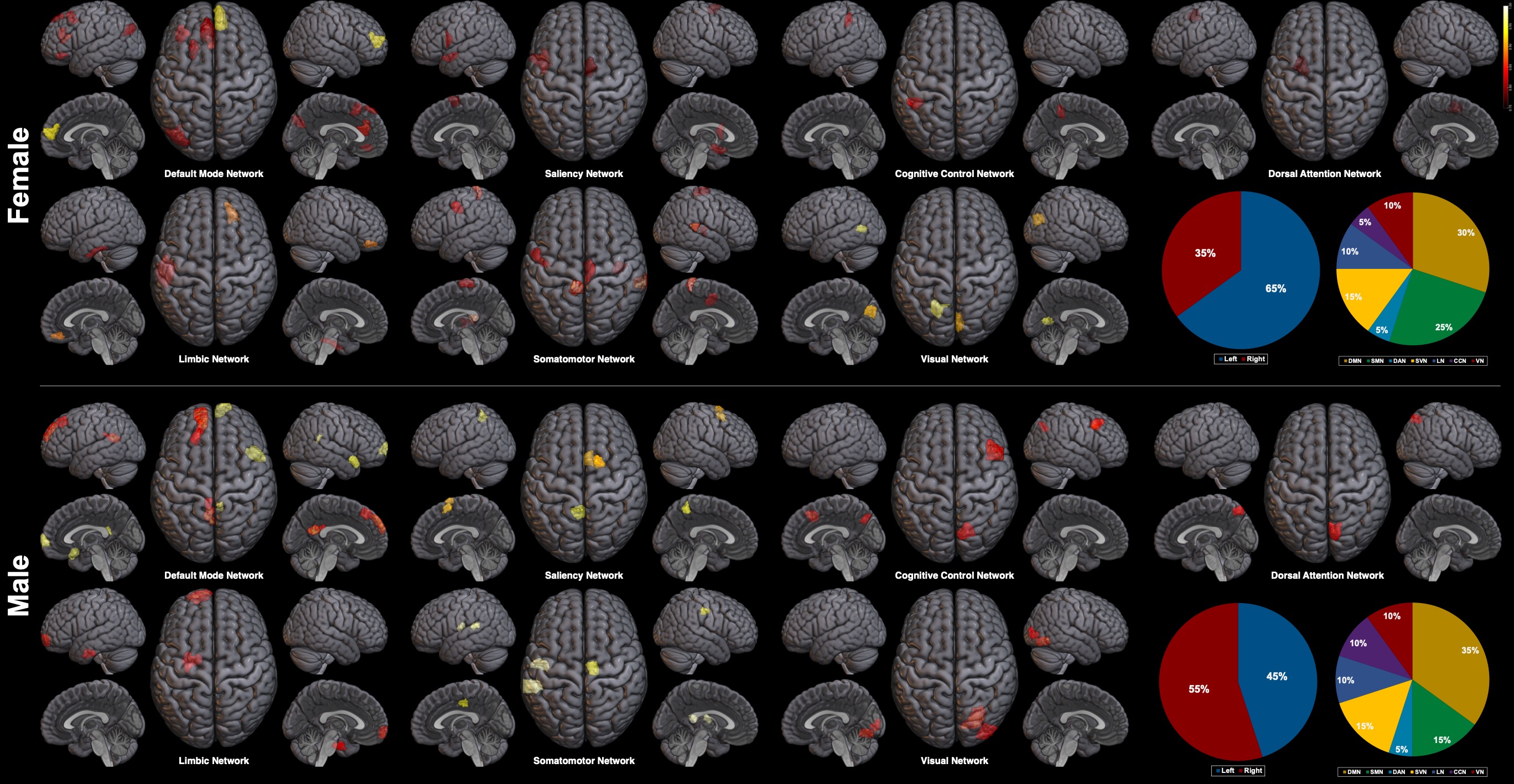}
\end{center}
\caption{Saliency mapping result of the proposed method. Top 20 salient regions are plotted with respect to the Yeo 7 networks \citep{thomas2011organization}. The pie charts indicate the ratio of the two hemispheres and the ratio of each networks across the salient regions.}
\label{fig:saliency_proposed}
\end{figure}
\end{landscape}

\begin{landscape}
\begin{figure}[ht]
\begin{center}
\includegraphics[width=1.35\textwidth,trim=4 4 4 4,clip]{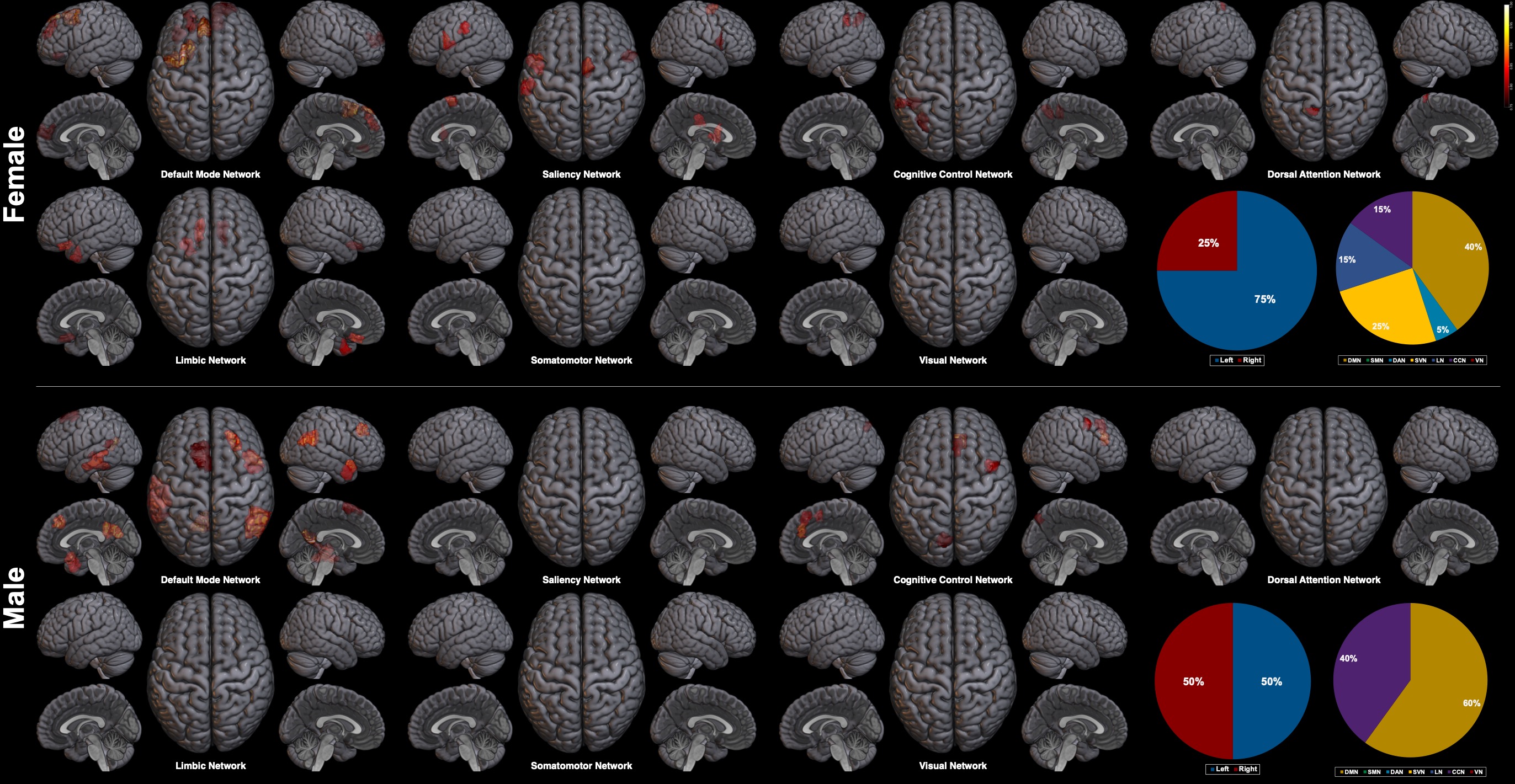}
\end{center}
\caption{Saliency mapping result of the CAM-based method. The pie charts indicate the ratio of the two hemispheres and the ratio of each networks across the salient regions.}
\label{fig:saliency_cam}
\end{figure}
\end{landscape}

\end{document}